\documentclass[letterpaper, 10 pt, conference]{ieeeconf}

\IEEEoverridecommandlockouts
\newif\ificraanonymous
\icraanonymousfalse
\ificraanonymous
  \title{%
AC-DC: Adaptive Communication for Scalable Dynamic Average Consensus in Multi-Robot Ergodic Search
}

\author{Anonymous Authors}

\else
  \makeatletter
\newcommand{\IEEEauthorrefmark}[1]{\textsuperscript{\scriptsize #1}}
\makeatother

\title{

AC-DC: Adaptive Communication for Scalable Dynamic Average Consensus in Multi-Robot Ergodic Search

}

\author{Robin Inho Kee\IEEEauthorrefmark{1,4}, Begum Cannataro\IEEEauthorrefmark{2}, Vasileios Tzoumas\IEEEauthorrefmark{3}%
\thanks{\IEEEauthorrefmark{1}Robin Inho Kee is with the Department of Robotics, University of Michigan, Ann Arbor, MI, USA. {\tt\small inhokee@umich.edu}}%
\thanks{\IEEEauthorrefmark{2}Begum Cannataro is with the Charles Stark Draper Laboratory, Cambridge, MA, USA. {\tt\small bcannataro@draper.com}}%
\thanks{\IEEEauthorrefmark{3}Vasileios Tzoumas is with the Department of Aerospace Engineering, University of Michigan, Ann Arbor, MI, USA. {\tt\small vtzoumas@umich.edu}}%
\thanks{\IEEEauthorrefmark{4}Robin Inho Kee is a Draper Scholar with the Charles Stark Draper Laboratory, Cambridge, MA, USA. The authors thank the Draper Scholars program for supporting this work.}%
}

\fi

\usepackage[T1]{fontenc}
\usepackage{cite}
\usepackage{amsmath,amssymb,mathtools,bm}

\usepackage{amsthm}

\usepackage{graphicx}
\usepackage{booktabs,multirow,array,tabularx}
\usepackage{float}
\usepackage{placeins}
\usepackage{algorithm}
\usepackage[noend]{algpseudocode}
\makeatletter
\renewcommand{\theALG@line}{\thealgorithm.\arabic{ALG@line}}
\providecommand*{\theHALG@line}{}
\renewcommand*{\theHALG@line}{\thealgorithm.\arabic{ALG@line}}
\makeatother
\usepackage{siunitx}
\usepackage{xcolor}
\usepackage{xspace}
\usepackage{microtype}
\let\labelindent\relax
\usepackage{enumitem}
\usepackage{tikz}
\usetikzlibrary{arrows.meta,positioning,fit,calc}
\usepackage{url}

\makeatletter
\let\NAT@parse\undefined
\makeatother
\usepackage[bookmarks=false]{hyperref}
\usepackage{cleveref}

\hypersetup{%
  colorlinks=true,
  linkcolor=blue,
  filecolor=magenta,
  urlcolor=black,
  citecolor=red,
  linkbordercolor={0 0 1}
}

\usepackage{etoolbox}
\makeatletter
\patchcmd{\@makecaption}
  {\begin{center}{\footnotesize #1}\\{\footnotesize\scshape #2}\end{center}}
  {\parbox{\hsize}{\footnotesize\noindent #1:\ #2\par}}
  {}
  {\PackageWarning{table-caption}{Could not patch table caption format}}
\makeatother

\crefname{algorithm}{Algorithm}{Algorithms}
\Crefname{algorithm}{Algorithm}{Algorithms}

\graphicspath{{./figures/}}
\allowdisplaybreaks
\setlist[itemize]{leftmargin=*,topsep=2pt,itemsep=1pt,parsep=0pt}
\setlist[enumerate]{leftmargin=*,topsep=2pt,itemsep=1pt,parsep=0pt}

\newtheorem{problem}{Problem}

\newtheorem{proposition}{Proposition}

\newtheorem{remark}{Remark}

\newcommand{\bdmath}{\begin{dmath}}
\newcommand{\edmath}{\end{dmath}}
\newcommand{\beq}{\begin{equation}}
\newcommand{\eeq}{\end{equation}}
\newcommand{\bdm}{\begin{displaymath}}
\newcommand{\edm}{\end{displaymath}}
\newcommand{\bea}{\begin{eqnarray}}
\newcommand{\eea}{\end{eqnarray}}
\newcommand{\beal}{\beq \begin{array}{lll}}
\newcommand{\eeal}{\end{array} \eeq}
\newcommand{\beas}{\begin{eqnarray*}}
\newcommand{\eeas}{\end{eqnarray*}}
\newcommand{\ba}{\begin{array}}
\newcommand{\ea}{\end{array}}
\newcommand{\bit}{\begin{itemize}}
\newcommand{\eit}{\end{itemize}}
\newcommand{\ben}{\begin{enumerate}}
\newcommand{\een}{\end{enumerate}}

\DeclareMathOperator*{\argmax}{arg\,max}

\DeclareRobustCommand{\scenario}[1]{%
  \ifmmode
    \mathsf{#1}%
  \else
    \texorpdfstring{{\fontsize{8.9}{9}\selectfont\sffamily #1}\xspace}{#1}%
  \fi
}

\newcommand{\red}[1]{{\color{red}#1}}
\newcommand{\blue}[1]{{\color{blue}#1}}

\newcommand{\myprob}{Problem~\ref{prob:team_tracking}\xspace}

\newif\ifshowcomments
\showcommentsfalse

\ifshowcomments
\else
  \renewcommand{\red}[1]{}
  \renewcommand{\blue}[1]{}
\fi

\makeatletter
\def\@IEEEfigurecaptionsepspace{\vspace{-5.00pt}}

\def\section{\@startsection{section}{1}{\z@}%
  {0.35ex plus 0.15ex minus 0.1ex}%
  {0.18ex plus 0.08ex minus 0.08ex}%
  {\normalfont\normalsize\centering\scshape}}
\def\subsection{\@startsection{subsection}{2}{\z@}%
  {0.35ex plus 0.15ex minus 0.1ex}%
  {0.18ex plus 0.08ex minus 0.08ex}%
  {\normalfont\normalsize\itshape}}

\def\thm@space@setup{\thm@preskip=1.5pt plus 0.6pt minus 0.4pt%
  \thm@postskip=1.5pt plus 0.6pt minus 0.4pt}

\patchcmd{\thebibliography}{\itemsep 0pt plus .5pt\relax}%
  {\itemsep -0.90pt plus .15pt\relax}{}%
  {\PackageError{layout-spacing}{Bibliography spacing patch failed}{}}
\patchcmd{\thebibliography}%
  {\vskip 0.3\baselineskip plus 0.1\baselineskip minus 0.1\baselineskip}%
  {\vskip 0.5pt}{}%
  {\PackageError{layout-spacing}{Bibliography heading spacing patch failed}{}}
\makeatother

\AtBeginEnvironment{proof}{%
  \pretocmd{\trivlist}{\topsep1.5pt plus 0.4pt minus 0.4pt\relax}{}%
    {\PackageError{layout-spacing}{Proof spacing hook failed}{}}}

\begin{document}
\maketitle
\thispagestyle{empty}
\pagestyle{empty}

\begin{abstract}
We study scalable peer-to-peer dynamic average consensus (DC) for multi-robot systems under finite-range, finite-rate, and interference-constrained communication. We introduce Adaptive Communication for Dynamic Average Consensus (AC-DC), which jointly adapts \emph{Who} communicates with whom, \emph{When}, and over \emph{What} parts of the consensus state, using local inputs and successfully received neighbor information.
Each robot's consensus state estimates the current average of the robots' local inputs. AC-DC updates these estimates as local inputs change and averages the values exchanged between robot pairs.
In AC-DC, robot pairs update without waiting for every robot to complete a communication round, and the selected-state messages carry consensus state coordinates independent of team size for a fixed state representation.
We apply AC-DC to dynamic-priority multi-robot ergodic search: one consensus stream estimates team visitation for motion coordination, while the other fuses regional measurement information to update uncertainty maps and search targets. Across twelve settings with up to 80 robots and 20 paired trials per setting, AC-DC has the lowest mean (i)~normalized covariance-trace area under the curve (AUC) and (ii)~attempted modeled communication payload among the compared decentralized methods. Averaged across settings, AC-DC achieves paired AUC reductions of $27.5\%$ relative to state-of-the-art baselines, with $8.7\times$ less communication traffic. As the number of robots increases, we observe that AC-DC's communication payload approaches that of the ideal centralized baseline (one ground compute-station communicating directly with all robots): with $120$ robots in a fixed $600\times600$~m scaling test, AC-DC uses $19.3$~MB versus $19.2$~MB for the ideal centralized baseline, while remaining peer-to-peer.
\end{abstract}

\section{Introduction}\label{sec:introduction}

Future multi-robot teams will perform information-gathering tasks such as mapping, search, inspection, and active information acquisition~\cite{xu2025communication,atanasov2015decentralized,corah2018distributed}. In large or unknown environments, communication infrastructure may be unavailable or unreliable. The DARPA Subterranean Challenge, for example, combined large unknown environments with degraded communication, while systems such as ACHORD prioritized mission information under intermittent connectivity~\cite{darpa2021subt,saboia2022achord,kottege2024heterogeneous}. Such settings require robots to keep evolving team information current while they move and sense.

Ideally, these teams coordinate peer-to-peer, avoiding dependence on a single station that may be unreachable, fail, or be attacked. Yet peer-to-peer coordination can become communication-heavy as teams grow: repeated rounds and large state exchanges consume airtime and delay decisions. We study this problem through dynamic average consensus (DC), where each robot's consensus state estimates the current average of the robots' time-varying local inputs~\cite{kia2015dynamic,zhu2010discrete,kia2019tutorial}. We seek a DC policy that adapts communication using received neighbor information, with selected-state message dimension independent of team size for a fixed state representation.

This motivates our central question: under finite-range, finite-rate, and interference-constrained communication, how can each robot jointly decide \emph{Who} communicates with whom, \emph{When}, and over \emph{What} parts of the consensus state to efficiently and effectively track the time-varying team averages?
 
The wireless medium couples the three decisions: finite data rate makes message size determine serialization time~\cite{kurose2020computer}, while finite range and receiver-side interference restrict concurrent exchanges~\cite{wang2006interference_tdma,gandham2008link_scheduling}. Failed or delayed deliveries leave stale neighbor information, coupling state-block selection, communication triggering, and neighbor selection.

The gap is to adapt these three decisions jointly using received neighbor information over time-varying graphs under the radio constraints above.
DC methods track moving averages under specified communication laws~\cite{kia2015dynamic,zhu2010discrete,kia2019tutorial,bastianello2022admm_dac_imperfect}.

For \emph{Who}, gossip selects edges randomly~\cite{boyd2006randomized} or by state disagreement~\cite{ustebay2010greedy}; state-dependent selection also appears in distributed convex optimization~\cite{verma2023maximal}. In AC-DC, we rank neighbors for the selected block using stored neighbor blocks and their change summaries, allowing us to assess exchanges before receiving current block values.

For \emph{When}, triggered communication adapts transmission timing in dynamic average consensus~\cite{kia2015distributed} and decentralized online convex optimization~\cite{cao2020decentralized}.
We trigger requests using the selected block's exchange score, which balances predicted disagreement reduction against block size to limit exchanges with little expected benefit.
If no neighbor meets the score threshold, we use service age to request overdue refreshes.

For \emph{What}, block-gossip methods group graph edges or subgraphs for averaging~\cite{haddock2022paving}, while dynamic compression changes the message representation in average consensus~\cite{makridis2025average}.
We select the \emph{state block}---a subset of entries in a robot's consensus vector---that we expect to be most useful to exchange, based on local input changes and disagreement with stored neighbor blocks. In selected-block mode, we send only its entries to reduce the data per exchange.
Related multi-robot methods improve efficiency and effectiveness using local mesh neighborhoods and adapting them online~\cite{xu2025communication,xu2026meshnet}. These works address distributed submodular optimization over discrete decision domains, whereas AC-DC addresses DC over continuous-valued signals.

To address this gap, we provide Adaptive Communication for Dynamic Average Consensus (AC-DC). 
AC-DC partitions each consensus state into \emph{state blocks} and stores the last block successfully received from each recently heard neighbor. \emph{Drift summaries} describe changes in the neighbor's current block relative to the last shared block values, without retransmitting the block. The drift summaries and stored blocks support a \emph{candidate score} that estimates the value of exchanging a particular state block with that neighbor.
We combine local input changes and stored neighbor information into a \emph{block candidate score} to decide \emph{What} to share. We request the block \emph{When} a neighbor meets the candidate-score threshold or, if none qualifies, is overdue for refresh. Then, the highest-scoring qualified neighbor determines \emph{Who} to contact.

We initialize each consensus state at its local input, add local input changes, and average values exchanged by communicating pairs while preserving their sum~\cite{zhu2010discrete,kia2019tutorial,boyd2006randomized}. These updates keep the average consensus state equal to the current average input. To preserve the pair sum, both robots update only after both directions of the exchange succeed.

We instantiate AC-DC in dynamic-priority multi-robot ergodic search. 
Ergodic control seeks to distribute team visitation in proportion to a \emph{spatial target}, a probability distribution---over the search domain--- that represents the localization uncertainty of points of interest~\cite{mathew2011metrics,mavrommati2018realtime}, e.g., detected landmarks. This distribution allocates more search effort to regions with greater remaining uncertainty and mission priority. 
Our formulation uses two DC streams: a \emph{trajectory stream} estimates the team's past and planned visitation, enabling each robot to coordinate its motion using team coverage information~\cite{abraham2018decentralized};
a \emph{sensing-information stream} fuses regional measurement-information contributions to update each robot's uncertainty map, which we combine with mission priority to construct its spatial target. All robots receive the mission-priority command directly.

\textbf{Contributions.} This paper makes three contributions.
\begin{enumerate}[label=\textbf{\arabic*)},leftmargin=5.4mm,itemsep=2pt,topsep=2pt]
    \item \textbf{Adaptive communication for DC.} 
    We introduce a communication-efficient robot-local policy that jointly adapts \emph{Who} communicates with whom, \emph{When}, and over \emph{What} parts of the consensus state, using local inputs, received neighbor blocks, and their drift summaries.
    \item \textbf{Scalable communication structure and complexity.} 
    We prioritize changing or disagreeing blocks for transmission in selected-block mode, using roughly $1/10$ to $1/42$ as many state entries as a full stream in our experiments.
    We bound per-frame communication, including initialization and metadata. We compare with two approaches to constrained communication for DC:
    alternating direction method of multipliers for dynamic average consensus (ADMM-DAC), which works asynchronously by sending neighbor-specific full-state messages~\cite{bastianello2022admm_dac_imperfect},
    and push-pull average consensus with dynamic compression (PP-ACDC), which sends compressed messages~\cite{makridis2025average}.
    We evaluate both on the same time-varying inputs.
    \item \textbf{Two-stream decentralized ergodic search.} 
    We use AC-DC in our search formulation to estimate team trajectory statistics for motion coordination and fuse regional measurement information for uncertainty-map updates, from which each robot forms its search target.
\end{enumerate}

\textbf{Evaluations.} 
Across twelve settings with up to 80 robots and 20 paired trials per setting, AC-DC has the lowest mean normalized covariance-trace AUC and attempted modeled communication payload among the compared decentralized methods. The mean paired AUC reductions are $29.5\%$ relative to ADMM-DAC and $24.5\%$ relative to PP-ACDC; these baselines use $5.8\times$ and $11.5\times$ more payload than AC-DC. In fixed-area scalability tests, the modeled-payload gap to the ideal centralized baseline narrows to $0.4\%$ at $N=120$, where AC-DC used $19.3$~MB versus $19.2$~MB.

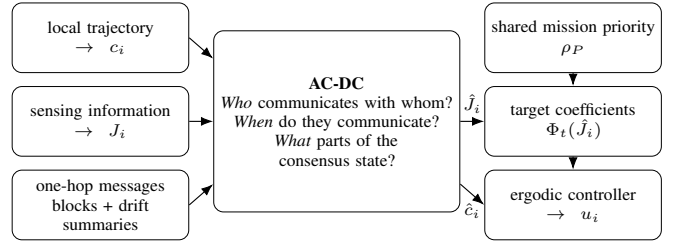
\begin{figure}[t]
\centering
\resizebox{\columnwidth}{!}{%
\begin{tikzpicture}[
    x=1cm,y=1cm,
    >=Latex,
    font=\scriptsize,
    box/.style={draw=black,text=black,rounded corners,align=center,inner sep=2pt,
        text width=24mm,minimum height=10mm},
    arrow/.style={->,draw=black,line width=0.45pt}
]
\node[box] (traj) at (0,1.2)
{local trajectory\\$\rightarrow c_i$};
\node[box] (sense) at (0,0)
{sensing information\\$\rightarrow J_i$};
\node[box] (nbr) at (0,-1.2)
{one-hop messages\\blocks + drift summaries};

\node[box,text width=34mm,minimum height=26mm] (acdc) at (3.4,0)
{\textbf{AC-DC}\\
\emph{Who} communicates with whom?\\
\emph{When} do they communicate?\\
\emph{What} parts of the consensus state?};

\node[box] (priority) at (6.8,1.2)
{shared mission priority\\$\rho_P$};
\node[box] (target) at (6.8,0)
{target coefficients\\$\Phi_t(\hat J_i)$};
\node[box] (ctrl) at (6.8,-1.2)
{ergodic controller\\$\rightarrow u_i$};

\draw[arrow] (traj.east) -- ([yshift=9mm]acdc.west);
\draw[arrow] (sense.east) -- (acdc.west);
\draw[arrow] (nbr.east) -- ([yshift=-9mm]acdc.west);
\draw[arrow] (acdc.east) -- node[above,text=black] {$\hat J_i$} (target.west);
\draw[arrow] (priority.south) -- (target.north);
\draw[arrow] (target.south) -- (ctrl.north);
\draw[arrow] ([yshift=-9mm]acdc.east) -- node[below,text=black] {$\hat c_i$} (ctrl.west);
\end{tikzpicture}}%
\caption{AC-DC communication layer for multi-robot ergodic search. We use local inputs, stored neighbor blocks, and drift summaries to decide \emph{Who} communicates with whom, \emph{When}, and over \emph{What} parts of the consensus state. The trajectory estimate $\hat c_i$ enters the controller directly; the sensing-information estimate $\hat J_i$ and mission priority give target coefficients $\Phi_t(\hat J_i)$.}
\label{fig:system_pipeline}
\end{figure}
\section{Problem Formulation}\label{sec:problem}

We consider $N$ robots performing decentralized ergodic information gathering using finite-range, finite-rate, half-duplex wireless links with receiver-side interference.

\noindent\textbf{Robot Motion.}\label{subsec:robots_sensing}
Let $\mathcal V\triangleq\{1,\ldots,N\}$ index robots in the bounded search domain $\mathcal D\subset\mathbb R^2$. Robot $i$ has state $x_i(t)$, position $p_i(t)\in\mathcal D$, and control input $u_i(t)$. We use control-affine dynamics as in receding-horizon ergodic control~\cite{mavrommati2018realtime}:
\begin{equation}
\dot x_i(t)=f_i\!\left(x_i(t)\right)+g_i\!\left(x_i(t)\right)u_i(t).
\label{eq:robot_dynamics}
\end{equation}
Control inputs are updated every $\Delta t>0$.

\noindent\textbf{Sensing.}
We use a linear-Gaussian sensing model~\cite{atanasov2015decentralized}. Let $\mathcal Q\subset\mathcal D$ be the set of map-cell center locations, with common cell area $\Delta A$, and let $\theta$ denote the map state. Each $r\in\mathcal Q$ indexes the map cell centered at location $r$. A measurement acquired by robot $i$ at position $p_i(t_\nu)$ is
\begin{equation}
z_{i,\nu}=H_i\!\left(p_i(t_\nu)\right)\theta+v_{i,\nu},
\;
v_{i,\nu}\sim\mathcal N\!\left(0,R_i\!\left(p_i(t_\nu)\right)\right).
\label{eq:sensing_information_model}
\end{equation}
Here $\mathcal N(0,R)$ denotes a zero-mean Gaussian distribution with covariance $R$, $R_i(p)\succ0$ is the measurement-noise covariance, and the measurement noises are independent across robots and sampling times. The measurement-information matrix is $M_i(p)\triangleq H_i(p)^\top R_i(p)^{-1}H_i(p)$. If all measurements were pooled, independent information adds to the prior as~\cite{atanasov2015decentralized}
\begin{equation}
\begin{aligned}
\Omega(t)
&=\Omega_0+\sum_{i=1}^{N}\sum_{\nu=1}^{m_i(t)}M_i\!\left(p_i(t_\nu)\right),
\Sigma(t)=\Omega(t)^{-1},
\end{aligned}
\label{eq:local_information}
\end{equation}
where $\Sigma_0\succ0$, $\Omega_0=\Sigma_0^{-1}$, and $m_i(t)$ is the number of measurements acquired by robot $i$ from mission start through time $t$. Under our independent-cell specialization of this additive information update, robot $i$'s accumulated cell-wise measurement information is
$I_i(t,r)\triangleq\sum_{\nu=1}^{m_i(t)}[M_i(p_i(t_\nu))]_{rr}$. The posterior covariance therefore depends on sensing geometry and noise statistics through the accumulated information.

\noindent\textbf{Communication Model.}\label{subsec:communication_model}
Communication-range adjacency is represented by the undirected graph $\mathcal G(t)=(\mathcal V,\mathcal E(t))$:
\begin{equation}
\{i,j\}\in\mathcal E(t)
\quad\Longleftrightarrow\quad
\|p_i(t)-p_j(t)\|_2\le R_c,
\label{eq:range_graph}
\end{equation}
where $R_c>0$ is the communication radius; $\{i,j\}$ denotes mutual range adjacency and $(i,j)$ denotes a directed transmission. Robots infer local neighbor availability from successfully delivered messages rather than observing $\mathcal G(t)$ globally.

The shared radio is half-duplex, and nearby simultaneous transmitters may interfere at a receiver. Directed transmissions $(i,j)$ and $(i',j')$ conflict if
\begin{equation}
\begin{aligned}
&\{i,j\}\cap\{i',j'\}\ne\emptyset,\quad\qquad\text{or}\\[-1mm]
&\|p_{i'}(t)-p_j(t)\|_2\le R_I,\!\!\!\quad\quad\text{or}\\[-1mm]
&\|p_i(t)-p_{j'}(t)\|_2\le R_I,
\end{aligned}
\label{eq:link_conflict}
\end{equation}
where $R_I>0$ is the interference radius~\cite{wang2006interference_tdma,gandham2008link_scheduling}; the first condition enforces half-duplex operation and the latter two cross-link interference.

An $L$-byte packet transmitted at rate $R_{\mathrm{bit}}$ occupies the channel for time equal to~\cite{kurose2020computer}
\begin{equation}
\tau(L)=\frac{8L}{R_{\mathrm{bit}}}.
\label{eq:packet_duration}
\end{equation}

\noindent\textbf{Regional Sensing-Information State.}\label{subsec:map_information}
To obtain a fixed-dimensional communication state, we use a finite-dimensional spatial-field representation as in distributed sensing and coverage~\cite{schwager2009decentralized,elwin2020distributed}. Viewing $I_i(t,\cdot)$ as piecewise constant over map cells, let $\psi_\ell(r)=\mathbf 1_{\mathcal Q_\ell}(r)$ denote the indicators of a fixed partition $\{\mathcal Q_\ell\}_{\ell=1}^{K_J}$ of the map cells into nonoverlapping regions. Under a common sensing schedule where $m_i(t)=m(t)$, for $m(t)\ge1$ define
\begin{equation}
[J_i(t)]_\ell
=\frac{1}{m(t)}\int_{\mathcal D} I_i(t,r)\psi_\ell(r)\,dr,
\; \ell=1,\ldots,K_J.
\label{eq:regional_information}
\end{equation}
For $m(t)=0$, set $J_i(t)=0$. Each entry of $J_i(t)$ summarizes how much sensing information robot $i$ has accumulated in one spatial region; the fixed-dimensional vector lets robots communicate where information has been gathered without transmitting the full cell-wise field. Since $m(t)$ is common to all robots, the normalization removes only the shared accumulation scale; with $\bar J(t)\triangleq N^{-1}\sum_iJ_i(t)$, $Nm(t)\bar J(t)$ recovers the pooled regional information masses.

\noindent\textbf{Ergodic Search Interface.}\label{subsec:ergodic_consensus}
The sensing posterior indicates where uncertainty remains, while the shared mission-priority map specifies where reducing it matters most. Together, they define the spatial target distribution used for information gathering. Ergodic control then drives the team's trajectory visitation toward this target so that measurements are collected in regions where more information is needed~\cite{miller2016ergodic}. 

Ergodic control represents this target and the team trajectory in Fourier space~\cite{mathew2011metrics,mavrommati2018realtime}. Let $\mathcal K\subset\mathbb N_0^2$ index the retained Fourier modes, with $K_c\triangleq|\mathcal K|$ and spectral weights $\Lambda_\kappa$. Let $c_i(t)\in\mathbb R^{K_c}$ denote the Fourier coefficients of robot $i$'s trajectory-visitation statistics~\cite{mathew2011metrics,mavrommati2018realtime}, and let $\bar c(t)\triangleq N^{-1}\sum_i c_i(t)$ be their team average.

Let $\rho_P(t,r)\ge0$ be a mission-priority map that may change by mission phase. Given a regional sensing-information state $J$, the decoded information-form posterior and $\rho_P(t,\cdot)$ define the normalized spatial target $\rho_t(J,\cdot)$.

For the normalized Fourier basis $F_\kappa$, define the target-coefficient vector $\Phi_t(J)\in\mathbb R^{K_c}$ by
\begin{equation}
[\Phi_t(J)]_\kappa
=\int_{\mathcal D} F_\kappa(r)\rho_t(J,r)\,dr,
\qquad \kappa\in\mathcal K.
\label{eq:effective_target_coefficients}
\end{equation}
The ergodic objective penalizes the weighted mismatch between team trajectory and current target coefficients:
\begin{equation}
\mathcal E_{\mathrm{erg}}(t)
\triangleq
\sum_{\kappa\in\mathcal K}\Lambda_\kappa
\left(\bar c_\kappa(t)-[\Phi_t(\bar J(t))]_\kappa\right)^2.
\label{eq:ergodic_mismatch}
\end{equation}
The fixed receding-horizon controller in~\cite{mavrommati2018realtime} reduces this mismatch subject to the robot dynamics; particularly, with globally available team quantities, its interface is
\begin{equation}
u_i(t)=\pi_{\mathrm{erg}}\!\left(x_i(t);\bar c(t),\Phi_t\!\left(\bar J(t)\right)\right),
\qquad i\in\mathcal V.
\label{eq:ergodic_controller_interface}
\end{equation}

In decentralized execution, neither team average is globally available: because $c_i(t)$ and $J_i(t)$ evolve with motion and sensing, robots must track $\bar c(t)$ and $\bar J(t)$ as two dynamic average consensus (DC) tasks~\cite{zhu2010discrete,kia2019tutorial}.

\begin{problem}[Communication-constrained decentralized ergodic search] \label{prob:team_tracking}
Under the radio models and fixed controller above, design a communication policy for each robot to track the moving team averages $\bar c(t)$ and $\bar J(t)$ using only its local inputs and successfully delivered one-hop information.
The policy jointly decides Who communicates with whom, When, and over What parts of the consensus state.
\end{problem}

\begin{remark}[Generality]
The DC backbone in \Cref{subsec:dynamic_input_tracking} applies to any vector-valued local input whose moving team average must be tracked; here, $c_i$ and $J_i$ instantiate the trajectory and sensing-information streams.
\end{remark}

\section{AC-DC: Adaptive Communication for Dynamic Average Consensus}\label{sec:abg_dac}\label{sec:scheduling}

We present AC-DC for \myprob. 
Standard DC tracks moving averages under given interaction graphs and communication laws~\cite{kia2015dynamic,zhu2010discrete,kia2019tutorial,bastianello2022admm_dac_imperfect}.
AC-DC instead jointly adapts \emph{Who} communicates with whom, \emph{When}, and over \emph{What} parts of the consensus state, using local inputs and neighbor information.

\subsection{Dynamic Average Consensus (DC) Backbone}\label{subsec:dynamic_input_tracking}
For stream $\mu\in\{c,J\}$, let $y_{i,c}=c_i$ and $y_{i,J}=J_i$, and let $K_\mu$ be the stream dimension. Robot $i$ maintains a consensus state $z_{i,\mu}\in\mathbb R^{K_\mu}$ initialized at its local input.

The local inputs change as the robots move and sense. Let $y_{i,\mu}^{\mathrm{old}}$ be the most recent input already incorporated into $z_{i,\mu}$. At each control update, standard dynamic average consensus adds the local input change~\cite{zhu2010discrete,kia2019tutorial},
\begin{equation}
\begin{aligned}
\Delta y_{i,\mu}&=y_{i,\mu}^{\mathrm{new}}-y_{i,\mu}^{\mathrm{old}},\\
z_{i,\mu}&\leftarrow z_{i,\mu}+\Delta y_{i,\mu},\qquad
y_{i,\mu}^{\mathrm{old}}\leftarrow y_{i,\mu}^{\mathrm{new}}.
\end{aligned}
\label{eq:abg_innovation}
\end{equation}
Communication decisions use this updated consensus state.


To reduce disagreement while preserving the pair sum, we use the standard pairwise gossip update~\cite{boyd2006randomized}.
If robots $i$ and $j$ successfully exchange the entries indexed by $\mathcal L\subseteq\{1,\ldots,K_\mu\}$, they update
\begin{equation}
\begin{aligned}
z_{i,\mu,\mathcal L}^{+}&=(1-\gamma)z_{i,\mu,\mathcal L}^{-}+\gamma z_{j,\mu,\mathcal L}^{-},\\
z_{j,\mu,\mathcal L}^{+}&=\gamma z_{i,\mu,\mathcal L}^{-}+(1-\gamma)z_{j,\mu,\mathcal L}^{-},
\end{aligned}
\qquad 0<\gamma<1.
\label{eq:pairwise_gossip}
\end{equation}
Here $-$ and $+$ denote the states before and after this \emph{gossip update}. This update preserves the pair sum and contracts the delivered-coordinate disagreement according to
\begin{equation}
\begin{aligned}
&\frac12\left\|z_{i,\mu,\mathcal L}^{-}-z_{j,\mu,\mathcal L}^{-}\right\|_2^2
-\frac12\left\|z_{i,\mu,\mathcal L}^{+}-z_{j,\mu,\mathcal L}^{+}\right\|_2^2\\
&\hspace{14mm}=2\gamma(1-\gamma)
\left\|z_{i,\mu,\mathcal L}^{-}-z_{j,\mu,\mathcal L}^{-}\right\|_2^2.
\end{aligned}
\label{eq:pair_disagreement_contraction}
\end{equation}
Together with the local input change update in \Cref{eq:abg_innovation} and initialization $z_{i,\mu}=y_{i,\mu}$, these pairwise-gossip properties preserve the moving network-average identity
$N^{-1}\sum_i z_{i,\mu}=N^{-1}\sum_i y_{i,\mu}$. Tracking accuracy still depends on which pair services are successfully delivered.

\Cref{alg:dac_overview} summarizes this backbone.

\begin{algorithm}[!ht]
\caption{Pairwise DC}\label{alg:dac_overview}
\footnotesize
\begin{algorithmic}[1]
\State \textbf{Input:}
a communication schedule specifying pair services $(i,j,\mu)$ and their transmission start times
\label{algline:dc_schedule}
\State Initialize $z_{i,\mu}\leftarrow y_{i,\mu}$ and $y_{i,\mu}^{\mathrm{old}}\leftarrow y_{i,\mu}$ for all $i$ and $\mu\in\{c,J\}$
\For{each control interval}
    \State Add $\Delta y_{i,\mu}$ to each consensus state and update $y_{i,\mu}^{\mathrm{old}}$ using \Cref{eq:abg_innovation}\label{algline:dc_input}
    \For{each pair service $(i,j,\mu)$ scheduled in this control interval}
    \label{algline:dc_pair}
        \State Exchange the complete stream $\mathcal L=\{1,\ldots,K_\mu\}$ between $i$ and $j$, beginning at the scheduled start time
        \label{algline:dc_exchange}
        \If{both directed state payloads are successfully delivered}
            \State Apply \Cref{eq:pairwise_gossip} on coordinates $\mathcal L$ \label{algline:dc_average}
        \EndIf
    \EndFor
\EndFor
\end{algorithmic}
\end{algorithm}

AC-DC replaces \Cref{alg:dac_overview}'s input schedule (Lines~\ref{algline:dc_schedule}, \ref{algline:dc_pair}--\ref{algline:dc_exchange}) with \Cref{alg:abg_step}. Each robot supplies local estimates $\hat c_i\triangleq z_{i,c}$ and $\hat J_i\triangleq z_{i,J}$ to the fixed ergodic controller, which evaluates \Cref{eq:ergodic_controller_interface} with $(\bar c,\bar J)$ replaced by $(\hat c_i,\hat J_i)$.

\subsection{Adaptive Communication (AC) Policy}\label{subsec:adaptive_communication}
AC-DC avoids repeatedly sending an entire consensus state by partitioning each stream into candidate \emph{state blocks}. 
Before block exchange, we initialize all neighbor blocks through a complete-stream first contact. 
We call subsequent bidirectional state exchanges \emph{ordinary services}.
For each recently heard neighbor, a robot stores the latest successfully delivered block and receives \emph{drift summaries} describing how the neighbor's current block has changed since it was last shared without retransmitting the block. Together they yield a \emph{predicted disagreement}, from which a \emph{candidate score} estimates the value of refreshing a block--neighbor pair. 
These quantities support the three adaptive decisions: \emph{Who} communicates with whom, \emph{When}, and over \emph{What} parts of the consensus state. Scheduling then resolves simultaneous requests, and only successful delivery commits the pairwise DC update.

Because the two DC streams $\{c,J\}$ share one radio timeline, a cross-stream airtime allocator activates one stream $\mu^\star\in\{c,J\}$ per communication frame. For each stream, robot $i$ maintains a directed HELLO cache $\mathcal H_{i,\mu}$ of recently heard senders, expiring after $\tau_H$
\Cref{alg:abg_step} summarizes the communication frame.

\begin{algorithm}[!ht]
\caption{AC-DC adaptive communication policy}\label{alg:abg_step}
\footnotesize
\begin{algorithmic}[1]
\State \textbf{Input:} active stream $\mu^\star$ selected by the cross-stream airtime allocator
\State Update $\mathcal H_{i,\mu^\star}$ from delivered HELLOs; expire entries after $\tau_H$ (\Cref{subsec:adaptive_communication})
\State Attempt complete-stream first contact for uninitialized pairs in either active-stream HELLO cache (\Cref{subsec:local_scheduling})
\State Update stored neighbor blocks and drift summaries from delivered messages (\Cref{subsec:local_scheduling}; \Cref{eq:directional_summary})
\For{each robot $i$ with an initialized current HELLO neighbor}
    \State Compute predicted disagreements and candidate scores using \Cref{eq:directional_discrepancy,eq:directional_utility}
    \State \textbf{\emph{What}:} select $b_i^\star$ and $\mathcal L_i^{\mathrm{tx}}$ using \Cref{eq:block_selection,eq:regime_message_coordinates}; set $\mathcal N_i^{\mathrm{sel}}\leftarrow\mathcal N_{i,\mu^\star,b_i^\star}^{\mathrm{mem}}$
    \State \textbf{\emph{When}:} form $\mathcal N_i^\star$ using \Cref{eq:when_trigger_sets,eq:when_candidates}
    \If{$\mathcal N_i^\star\neq\emptyset$}
        \State \textbf{\emph{Who}:} select the max-score $j_i^\star\in\mathcal N_i^\star$ and submit one request
    \EndIf
\EndFor
\State \textbf{Schedule:} form robot-disjoint pairs through request/acceptance and reserve TDMA slots (\Cref{subsec:abg_payload_complexity})
\State Attempt reserved exchanges; check delivery using \Cref{eq:range_graph,eq:link_conflict,eq:packet_duration}
\State For ordinary services delivered in both directions, apply \Cref{eq:pairwise_gossip}, refresh stored neighbor blocks, and reset their service ages
\end{algorithmic}
\end{algorithm}

\noindent\textbf{Communication blocks and stored neighbor blocks.}\label{subsec:local_scheduling}
The block partition defines the candidate \emph{What} choices and, in selected-block mode, the coordinates carried by an ordinary service. For nominal block length $K_\mu^{\mathrm{blk}}$,
\begin{equation}
\begin{aligned}
\mathcal B_\mu
&=\left\{0,\ldots,\left\lceil\frac{K_\mu}{K_\mu^{\mathrm{blk}}}\right\rceil-1\right\},\\
\mathcal L_{\mu,b}
&=\{bK_\mu^{\mathrm{blk}}+1,\ldots,
\min(K_\mu,(b+1)K_\mu^{\mathrm{blk}})\}.
\end{aligned}
\label{eq:block_partition}
\end{equation}
For a stream vector $v_{i,\mu}$, write $v_{i,\mu,b}=v_{i,\mu}[\mathcal L_{\mu,b}]$.

Write $\widetilde z_{j\to i,\mu,b}$ for robot $i$'s stored neighbor block from $j$.
Let $\mathcal N_{i,\mu,b}^{\mathrm{mem}}$ contain current HELLO-cache neighbors $j\in\mathcal H_{i,\mu}$ for which robot $i$ has initialized the stored neighbor block for block $b$. First contact initializes stored neighbor blocks in both directions without gossip; they persist through HELLO expiration but are ineligible for ordinary service until the sender is heard again. Successful ordinary services refresh stored neighbor blocks and reset their service ages. 

\noindent\textbf{Predicting neighbor disagreement.}\label{subsec:directional_utility}
As a neighbor's consensus state changes, its last successfully delivered block may become outdated. To decide whether block exchange is worthwhile, we send drift summaries to predict disagreement with the neighbor before retransmitting the block.

For the directed neighbor pair $j\to i$, let $z_{j\to i,\mu,b}^{\mathrm{snap}}$ be sender $j$'s block at their last successful ordinary service, and let $z_{j,\mu,b}^{\mathrm{frm}}$ be that block's value at the current frame's start. 
We summarize this change by its magnitude $\delta_{j\to i,b}$ and the sign $\sigma_{j\to i,b}$ of its squared-norm change, setting the sign to zero when the squared-norm change has magnitude at most $\epsilon_\mu$:
\begin{equation}
\begin{aligned}
\delta_{j\to i,b}
&=\left\|z_{j,\mu,b}^{\mathrm{frm}}-z_{j\to i,\mu,b}^{\mathrm{snap}}\right\|_2,\\
\Delta_{j\to i,b}^{\mathrm{rad}}
&=\left\|z_{j,\mu,b}^{\mathrm{frm}}\right\|_2^2
-\left\|z_{j\to i,\mu,b}^{\mathrm{snap}}\right\|_2^2,\\
\sigma_{j\to i,b}
&=\begin{cases}
0,&|\Delta_{j\to i,b}^{\mathrm{rad}}|\le\epsilon_\mu,\\
\operatorname{sgn}(\Delta_{j\to i,b}^{\mathrm{rad}}),&\text{otherwise}.
\end{cases}
\end{aligned}
\label{eq:directional_summary}
\end{equation}
We transmit $(\delta_{j\to i,b},\sigma_{j\to i,b})$ so the receiver can update predicted disagreement without block coordinates; $\Delta_{j\to i,b}^{\mathrm{rad}}$ is computed locally.
Using the latest delivered summary $(\widetilde\delta_{j\to i,b},\widetilde\sigma_{j\to i,b})$, robot $i$ computes the stored-block discrepancy and signed norm-drift adjustment:
\[
\begin{aligned}
d_{i\to j,b}^{\mathrm{mem}}
&=\|z_{i,\mu,b}-\widetilde z_{j\to i,\mu,b}\|_2,\\
\Delta_{i\to j,b}^{\mathrm{gap}}
&=\|z_{i,\mu,b}\|_2^2-\|\widetilde z_{j\to i,\mu,b}\|_2^2,\\
a_{i\to j,b}
&=\widetilde\sigma_{j\to i,b}\,
\operatorname{sgn}(\Delta_{i\to j,b}^{\mathrm{gap}}).
\end{aligned}
\]
The resulting predicted disagreement is
\begin{equation}
\widehat d_{i\to j,b}
=\Big[d_{i\to j,b}^{\mathrm{mem}}
+\beta_\mu a_{i\to j,b}\widetilde\delta_{j\to i,b}
-\epsilon_\mu\Big]_+.
\label{eq:directional_discrepancy}
\end{equation}
Here $\beta_\mu$ scales the drift-summary adjustment, and $\epsilon_\mu$ suppresses small predicted disagreements.
This is a local scoring proxy; the actual DC correction uses block values exchanged in a successful ordinary service.

\noindent\textbf{Scoring a candidate exchange.}
We score candidate exchanges by balancing predicted reductions in consensus disagreement against block size.
Substituting the predicted disagreement $\widehat d_{i\to j,b}$ from \Cref{eq:directional_discrepancy} into the decrease in \Cref{eq:pair_disagreement_contraction} gives $2\gamma(1-\gamma)\widehat d_{i\to j,b}^{\,2}$, which we can evaluate before receiving the neighbor's current block.

For the sensing-information stream ($\mu=J$), we favor exchanges involving blocks with a larger share of robot $i$'s estimated sensing information. We compute this share, $r_{i,b}^{J}$, from the positive entries of $z_{i,J}$ to keep the weight nonnegative under signed DC input updates.
We divide the sum of positive consensus coordinates in the block by the corresponding stream total plus $\varepsilon_{\mathrm{mass}}>0$, setting $r_{i,b}^{J}=0$ when that total is at most $\varepsilon_{\mathrm{mass}}$.
This weighting uses positive parts without altering the consensus state.
With stream scale $s_\mu>0$, information weight $\alpha_J$, and ranking-cost constants $c_0,c_1>0$, we define the candidate score
\begin{equation}
U_{i\to j,b}
=s_\mu\,
\frac{2\gamma(1-\gamma)
\left(1+\mathbf 1_{\{\mu=J\}}\alpha_J r_{i,b}^{J}\right)
\widehat d_{i\to j,b}^{\,2}}
{c_0+c_1|\mathcal L_{\mu,b}|}.
\label{eq:directional_utility}
\end{equation}
The ranking cost grows with the block's coordinate count, $|\mathcal L_{\mu,b}|$; \Cref{tab:protocol_accounting} accounts for full protocol traffic.
We use $U_{i\to j,b}$ to select neighbors and trigger communication in both streams, and as a term in the trajectory block candidate score.

\noindent\textbf{\emph{What} to share.}
\label{subsec:block_priorities}\label{subsec:message_size}
For the active stream---trajectory ($\mu^\star=c$) or sensing information ($\mu^\star=J$)---we choose the block index in $\mathcal B_{\mu^\star}$ with the largest \emph{block candidate score} $S_{i,b}^{\mu^\star}$:
\begin{equation}
b_i^\star\in\argmax_{b\in\mathcal B_{\mu^\star}}S_{i,b}^{\mu^\star}.
\label{eq:block_selection}
\end{equation}

\emph{Trajectory blocks ($c$).}
To weight trajectory changes consistently with the motion-coordination objective, we use the ergodic weights from \Cref{eq:ergodic_mismatch}. For a block change or disagreement vector $v_b$, the weighted sum of squared coordinates is
$\|v_b\|_{\Lambda,b}^2\triangleq v_b^\top\Lambda_bv_b$,
where $\Lambda_b$ is the diagonal matrix of these weights for block $b$.
We use it in the first three terms and the normalized candidate score in the fourth:
\begingroup
\setlength{\multlinegap}{0pt}
\begin{multline}
S_{i,b}^{c}=
\alpha_I^c\|\Delta y_{i,c,b}\|_{\Lambda,b}^2
+\alpha_M^c\max_{j\in\mathcal N_{i,c,b}^{\mathrm{mem}}}
\|z_{i,c,b}-\widetilde z_{j\to i,c,b}\|_{\Lambda,b}^2\\
+\alpha_Q^c\|z_{i,c,b}-z_{i,c,b}^{\mathrm{last}}\|_{\Lambda,b}^2
+\alpha_U^c\bar U_{i,b}.
\label{eq:c_block_score}
\end{multline}
\endgroup
The first term prioritizes recent changes in the local trajectory input, with $\Delta y_{i,c,b}$ given by \Cref{eq:abg_innovation}.
The second prioritizes the largest disagreement between the current trajectory consensus block $z_{i,c,b}$ and a stored neighbor block $\widetilde z_{j\to i,c,b}$. 
The third term prioritizes local consensus changes since the block's last successful ordinary service. Here $z_{i,c,b}^{\mathrm{last}}$ is the local block saved at that service, initialized to the initial block.
The fourth term favors high-scoring available exchanges: $\bar U_{i,b}$ is the largest $U_{i\to j,b}$ over eligible neighbors, normalized by the largest of these maxima across trajectory blocks.
Empty-neighbor maxima and zero-denominator ratios are zero.
We balance the four terms with $\alpha_I^c,\alpha_M^c,\alpha_Q^c,\alpha_U^c$.

\emph{Sensing-information blocks ($J$).}
We combine ranks of local input change and disagreement with stored neighbor blocks to balance these quantities despite their different scales.

Let $\Delta J_i^{\mathrm{sel}}$ be the change in $J_i$ since the preceding control update, held fixed throughout the current control interval and initialized to zero.
We denote the sum of its squared entries within block $b$ by $A_{i,b}^{J}$.
We denote by $D_{i,b}^{J}$ the largest sum of squared coordinate differences between robot $i$'s current sensing-information consensus block and an eligible stored neighbor block. We set $D_{i,b}^{J}=0$ if none is available.

Let $\operatorname{rank}_i(\cdot)\in[0,1]$ denote the normalized ascending rank across blocks.
With usable stored neighbor blocks, we define
\begin{equation}
\begin{aligned}
S_{i,b}^{J}
={}&\left(\max_{b'}A_{i,b'}^{J}\right)
\Big[(1-\lambda_J)\operatorname{rank}_i(A_{i,b}^{J})\\
&\qquad+\lambda_J\operatorname{rank}_i(D_{i,b}^{J})\Big].
\end{aligned}
\label{eq:J_block_score}
\end{equation}
We use $\lambda_J$ to balance the contribution of disagreement relative to local input change.
The common factor $\max_{b'}A_{i,b'}^{J}$ preserves block ordering when positive and gives zero scores when all local input changes vanish. Without usable stored neighbor blocks, we set $S_{i,b}^{J}=A_{i,b}^{J}$.
We resolve rank and block-selection ties deterministically, with rank zero for a single-block stream.

In both message modes, we use the selected block to trigger communication and choose a neighbor. The transmitted coordinate set $\mathcal L_i^{\mathrm{tx}}$ contains either all $K_{\mu^\star}$ coordinates of the active consensus stream or only those in the selected block:
\begin{equation}
\mathcal L_i^{\mathrm{tx}}=\begin{cases}
\{1,\ldots,K_{\mu^\star}\},&\text{complete-stream mode},\\
\mathcal L_{\mu^\star,b_i^\star},&\text{selected-block mode}.
\end{cases}
\label{eq:regime_message_coordinates}
\end{equation}

\noindent\textbf{\emph{When} to communicate.}\label{subsec:when_to_communicate}
To avoid low-value requests while allowing old blocks to be refreshed, we use a candidate-score threshold $\tau_{\mu^\star}^{U}$ and a service-age threshold $H_{\mu^\star}$. Let $h_{j\to i,b}^{\mathrm{svc}}$ count active-stream frames since $j\to i$ last served block $b$, and let $\mathcal N_i^{\mathrm{sel}}\triangleq\mathcal N_{i,\mu^\star,b_i^\star}^{\mathrm{mem}}$ contain the initialized current HELLO neighbors for the selected block. Define
\begin{equation}
\begin{aligned}
\mathcal N_i^{U}
&=\{j\in\mathcal N_i^{\mathrm{sel}}:
U_{i\to j,b_i^\star}\ge\tau_{\mu^\star}^{U}\},\\
\mathcal N_i^{H}
&=\{j\in\mathcal N_i^{\mathrm{sel}}:
h_{j\to i,b_i^\star}^{\mathrm{svc}}\ge H_{\mu^\star}\}.
\end{aligned}
\label{eq:when_trigger_sets}
\end{equation}
We first consider neighbors meeting the candidate-score threshold; if none qualifies, we use service age to request a refresh:
\begin{equation}
\mathcal N_i^\star=\begin{cases}
\mathcal N_i^{U},&\mathcal N_i^{U}\neq\emptyset,\\
\mathcal N_i^{H},&\text{otherwise}.
\end{cases}
\label{eq:when_candidates}
\end{equation}
A refresh still requires successful delivery.

\noindent\textbf{\emph{Who} to contact.}\label{subsec:with_whom}
If $\mathcal N_i^\star$ is nonempty, robot $i$ requests the qualified neighbor with the largest candidate score,
\begin{equation}
j_i^\star\in\argmax_{j\in\mathcal N_i^\star}U_{i\to j,b_i^\star}.
\label{eq:when_who}
\end{equation}
If $\mathcal N_i^\star$ is empty, robot $i$ submits no ordinary-service request.

\noindent\textbf{Scheduling and delivery.}\label{subsec:abg_payload_complexity} 
A requested exchange proceeds through request/acceptance and a locally feasible TDMA reservation \cite{gandham2008link_scheduling, ergen2010tdma}. We attempt the reserved exchanges under the range, interference, and packet-duration conditions in \Cref{eq:range_graph,eq:link_conflict,eq:packet_duration}.
To preserve the pair sum, we apply \Cref{eq:pairwise_gossip} and refresh stored neighbor blocks only after successful delivery in both directions.
Failed attempts consume modeled airtime and bytes but leave the consensus states unchanged.
First contact initializes stored neighbor blocks without gossip, while drift summaries update only the stored summary values. 
We use each slot’s initial consensus state; later slots reflect earlier successful updates.

\subsection{Communication Complexity}\label{subsec:communication_complexity}
For either DC stream $\mu\in\{c,J\}$, we next characterize the modeled communication generated in one frame. 

\begin{proposition}[Communication Complexity]\label{prop:communication_complexity}
Consider the attempted-on-air modeled bits generated by stream $\mu$ in communication frame $f$. 
Let $d_{\max}$ bound the number of active one-hop neighbor records at any robot, let $K_\mu^{\mathrm{blk}}$ be the maximum block length, and let $n_\mu^{\mathrm{fc}}(f)\le\lfloor N/2\rfloor$ be the number of first-contact pair exchanges in frame $f$. 
For the one-block selected mode used here and fixed summary/bundle limits, define the scalable-identifier metadata scaling term
\[
M_\mu\triangleq(1+d_{\max})\big(\log N+\log|\mathcal B_\mu|\big).
\]
With $O(1)$ bits per fixed-width state scalar, complete-stream traffic is $O(N(K_\mu+M_\mu))$, while selected-block traffic is
\begin{equation*}
O\!\left(
N(K_\mu^{\mathrm{blk}}+M_\mu)
+n_\mu^{\mathrm{fc}}(f)K_\mu
\right).
\end{equation*}
\end{proposition}

\begin{proof}
HELLO and summary records are aggregated per sender. 
Each robot issues at most one request and accept, while reservation and data exchanges form robot-disjoint matchings, giving $O(NM_\mu)$ metadata bits per frame. 
First-contact and ordinary state exchanges are jointly robot-disjoint, so at most $\lfloor N/2\rfloor$ state-carrying pair exchanges occur in a frame.
Complete-stream pair exchanges therefore contribute $O(NK_\mu)$ state bits. In selected-block mode, ordinary pair exchanges contribute $O(NK_\mu^{\mathrm{blk}})$ state bits, while the $n_\mu^{\mathrm{fc}}(f)$ first-contact exchanges contribute $O(n_\mu^{\mathrm{fc}}(f)K_\mu)$. 
Adding metadata gives the stated bounds.
\end{proof}


\begin{table}[!t]
\centering
\caption{AC-DC-specific parameters used in all settings.}
\label{tab:acdc_parameters}
\scriptsize
\setlength{\tabcolsep}{3pt}
\begin{tabular}{@{}ll@{}}
\toprule 
Quantity & Value \\
\midrule
$\gamma$ & $0.5$ \\
$K_c^{\mathrm{blk}}$ / $K_J^{\mathrm{blk}}$ & $40$ / $20$ \\
$(\beta_c,\beta_J)$ & $(0.75,0.1875)$ \\
$(\epsilon_c,\epsilon_J)$ & $(5\!\times\!10^{-5},2\!\times\!10^{-4})$ \\
$(s_c,s_J)$ & $(4,0.5)$ \\
$\alpha_J$ / $\varepsilon_{\mathrm{mass}}$ & $0.25$ / $10^{-12}$ \\
$(c_0,c_1)$ & $(12,4)$ \\
$(\tau_c^U,\tau_J^U)$ & $(10^{-7},2\!\times\!10^{-4})$ \\
$(H_c,H_J)$ & $(256,1024)$ \\
HELLO timeout $\tau_H$ [frames] & $3$ \\
$(\alpha_I^c,\alpha_M^c,\alpha_Q^c,\alpha_U^c)$ & $(0.7,0.2,0.1,0.1)$ \\
$\lambda_J$ & $0.25$ \\
\bottomrule
\end{tabular}
\end{table}

\begin{table}[!t]
\centering
\caption{Modeled AC-DC messages, in bytes (8 bits per byte).}
\label{tab:protocol_accounting}
\scriptsize
\setlength{\tabcolsep}{3pt}
\begin{tabular}{@{}lc@{}}
\toprule
Message or exchange & Modeled bytes per attempt \\
\midrule
HELLO & $24+4h_i+8s_i$ \\
First contact & $8K_\mu$ \\
Summary & $12+12e_i$ \\
Request & $32+8(q-1)$ \\
Accept & $16+4(q-1)$ \\
Reservation & $24+4(q-1)$ \\
Ordinary data & $8\sum_{\ell=1}^{q}k_\ell+16+12q\mathbf 1_{\{q>1\}}$ \\
\bottomrule
\end{tabular}
\par\smallskip
\begin{minipage}{\columnwidth}\scriptsize
$K_\mu$ is the stream dimension. The counts $h_i$ and $s_i$ are the numbers of 4-byte and 8-byte HELLO records, respectively; $e_i$ counts records in a separate summary message. A bundle contains $q$ blocks, and $k_\ell$ is the number of state coordinates in block $\ell$. Each state coordinate occupies 4 modeled bytes per direction, so a bidirectional full-stream exchange costs $8K_\mu$ bytes. The remaining terms account for metadata.
\end{minipage}
\end{table}

Because exactly one stream is active per communication frame, this is also the total per-frame modeled AC-DC communication bound. Mission-level traffic is obtained by summing these frame costs over the $c$- and $J$-stream frames.


In frames without first contact, our packet models give selected-block AC-DC the smallest communication bound when $K_\mu^{\mathrm{blk}}+M_\mu\ll K_\mu$. The bounds are $O(N(K_\mu^{\mathrm{blk}}+M_\mu))$ for AC-DC, $O(Nd_{\max}(K_\mu+\log N))$ for ADMM-DAC~\cite{bastianello2022admm_dac_imperfect}, which sends a full state to each active neighbor, and $O(N(K_\mu+\log N))$ for PP-ACDC~\cite{makridis2025average}, which broadcasts a full stream per active sender. We reduce the state payload from $K_\mu$ to $K_\mu^{\mathrm{blk}}$ entries and avoid ADMM-DAC's per-neighbor state replication, removing the $d_{\max}$ factor from our state term. PP-ACDC already avoids this replication, so our advantage over it comes from smaller blocks whose savings outweigh the additional metadata cost $M_\mu$ under the stated condition. Byte totals depend on packet sizes and attempted exchanges.


We use a full-stream first contact to initialize all stored neighbor blocks once per robot pair and stream. After this exchange succeeds, we retain the blocks for subsequent services, even through HELLO-cache expiration, so the initialization cost is not repeated. The $O(n_\mu^{\mathrm{fc}}(f)K_\mu)$ term in \Cref{prop:communication_complexity} therefore accounts for new initializations and retries after failed delivery. Although this cost can dominate a frame containing many first-contact attempts, subsequent services between initialized pairs incur only the selected-block and metadata costs above.


We incur the metadata cost to prioritize high-scoring exchanges and overdue refreshes, aiming to use communication where updating a block is most useful. In \Cref{tab:main_results}, see setting D1, our complete-stream mode achieves lower mean normalized covariance-trace AUC and attempted modeled payload than both evaluated full-stream baselines, demonstrating gains beyond selected-block transmission.

\section{Experiments}\label{sec:experiments}
We evaluate AC-DC in closed-loop dynamic-priority ergodic search against ADMM-DAC, PP-ACDC, and an ideal centralized baseline. The main study uses twelve environment--density settings up to $N=80$ with 20 paired Monte Carlo trials per setting, followed by fixed-area tests at $N=100, 120$. 
AC-DC has the lowest mean normalized covariance-trace AUC and attempted modeled payload among the compared decentralized methods. 
In the fixed-area scalability test, at $N=120$, AC-DC uses $19.3$~MB versus $19.2$~MB for the ideal centralized baseline while remaining peer-to-peer. 

\subsection{Experimental Setup}\label{subsec:experiment_setup}

\noindent\textbf{Scenarios and task.}
We use the reduced planar-quadrotor model based on~\cite{wu2016safety} and evaluate $300$, $450$, and $600$~m square environments with $N\in\{5,10,15,20\}$, $\{11,22,33,44\}$, and $\{20,40,60,80\}$ robots, respectively, defining approximately matched density regimes D1--D4. The stream dimensions are $K_c=K_J\in\{400,625,841\}$. Each paired trial uses common seeds across methods over $N_T=5000$ control intervals at $\Delta t=\SI{0.05}{s}$ ($T=\SI{250}{s}$), with one sensing update per interval. Initial positions are sampled uniformly with a $15$~m boundary margin and reused across methods.

\begin{figure}[!ht]
\centering
\includegraphics[width=0.82\columnwidth]{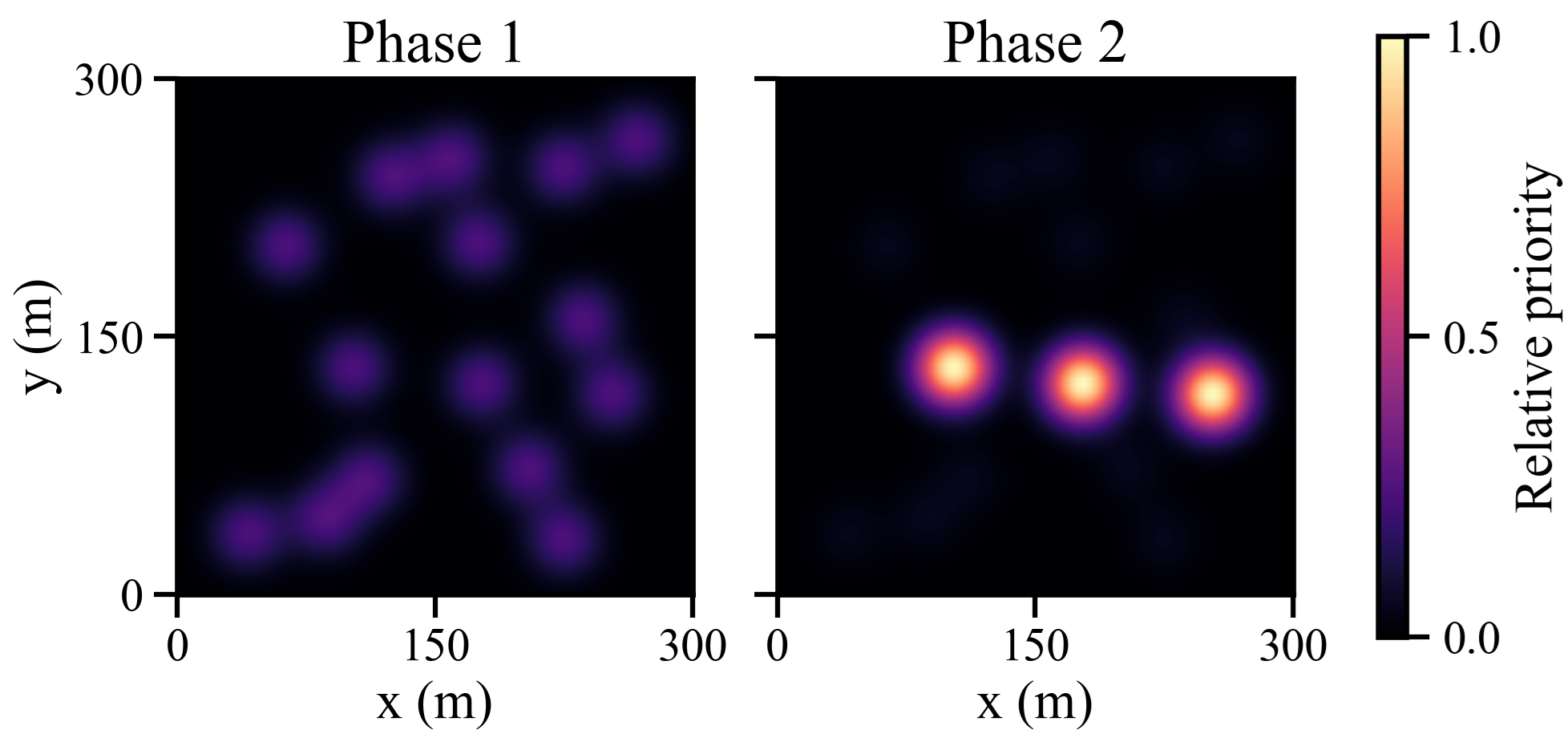}
\caption{$300\times300$~m dynamic-priority setup. The shared mission-priority map is multimodal in Phase~1 and switches at $t_c=\SI{100}{s}$ to the concentrated Phase~2 priority. The larger workspaces use the same two-phase task structure.}
\label{fig:dynamic_priority_setup}
\end{figure}


\begin{figure*}[!t]
\centering
\includegraphics[width=0.83\textwidth]{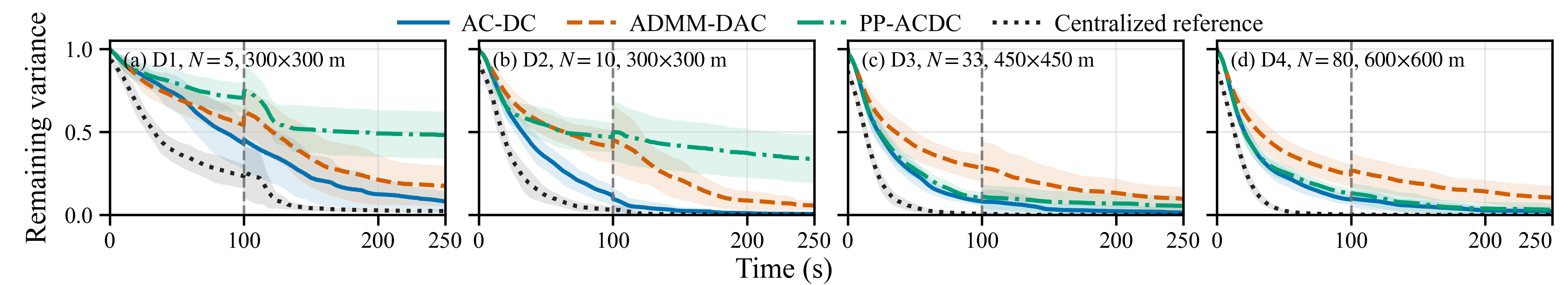}
\caption{Representative normalized covariance-trace ratio trajectories for D1--D4, 20 paired Monte Carlo trials (curves and shaded bands indicate the mean and $\pm$ sample standard deviation, respectively). The dashed line marks $t_c=\SI{100}{s}$; \Cref{tab:main_results} reports the 20-trial results for all settings.}
\label{fig:remaining_variance_all}
\end{figure*}

\begin{table*}[!t]
\centering
\caption{Closed-loop normalized covariance-trace AUC and attempted modeled protocol payload over
$250$~s (mean $\pm$ sample standard deviation, 20 paired Monte Carlo trials per setting). Lower is better; bold marks the best decentralized entry. Centralized denotes ideal centralized baseline, which uses ideal team information with separate robot-state uplink and control-input downlink accounting.}
\label{tab:main_results}
\scriptsize
\setlength{\tabcolsep}{1.65pt}
\begin{tabular}{@{}lcccccccc@{}}
\toprule
& \multicolumn{4}{c}{Normalized covariance-trace AUC}
& \multicolumn{4}{c}{Attempted modeled payload [MB]} \\
\cmidrule(lr){2-5}\cmidrule(l){6-9}
Setting
& AC-DC & ADMM-DAC & PP-ACDC & Centralized
& AC-DC & ADMM-DAC & PP-ACDC & Centralized \\
\midrule
300/N5 (D1)
& $\mathbf{0.415\pm0.110}$
& $0.489\pm0.105$
& $0.635\pm0.094$
& $0.218\pm0.034$
& $\mathbf{5.1\pm0.4}$
& $15.9\pm10.0$
& $27.2\pm15.8$
& $0.8\pm0.0$ \\

300/N10 (D2)
& $\mathbf{0.180\pm0.026}$
& $0.344\pm0.060$
& $0.468\pm0.100$
& $0.098\pm0.017$
& $\mathbf{4.0\pm0.1}$
& $31.8\pm14.1$
& $61.4\pm27.8$
& $1.6\pm0.0$ \\

300/N15 (D3)
& $\mathbf{0.136\pm0.031}$
& $0.231\pm0.089$
& $0.277\pm0.128$
& $0.064\pm0.011$
& $\mathbf{6.1\pm0.2}$
& $35.8\pm20.3$
& $90.7\pm40.8$
& $2.4\pm0.0$ \\

300/N20 (D4)
& $\mathbf{0.108\pm0.021}$
& $0.161\pm0.046$
& $0.168\pm0.052$
& $0.047\pm0.006$
& $\mathbf{9.1\pm0.3}$
& $55.3\pm27.1$
& $142.6\pm24.6$
& $3.2\pm0.0$ \\
\midrule

450/N11 (D1)
& $\mathbf{0.450\pm0.056}$
& $0.482\pm0.110$
& $0.592\pm0.088$
& $0.207\pm0.035$
& $\mathbf{7.9\pm0.3}$
& $12.9\pm9.8$
& $42.3\pm32.1$
& $1.8\pm0.0$ \\

450/N22 (D2)
& $\mathbf{0.239\pm0.049}$
& $0.364\pm0.067$
& $0.314\pm0.065$
& $0.093\pm0.017$
& $\mathbf{8.3\pm0.2}$
& $50.4\pm30.5$
& $74.3\pm64.5$
& $3.5\pm0.0$ \\

450/N33 (D3)
& $\mathbf{0.153\pm0.027}$
& $0.288\pm0.072$
& $0.189\pm0.044$
& $0.064\pm0.009$
& $\mathbf{11.2\pm0.6}$
& $66.2\pm46.5$
& $61.4\pm57.0$
& $5.3\pm0.0$ \\

450/N44 (D4)
& $\mathbf{0.126\pm0.031}$
& $0.207\pm0.058$
& $0.128\pm0.031$
& $0.047\pm0.006$
& $\mathbf{11.9\pm0.3}$
& $106.0\pm50.6$
& $100.4\pm60.2$
& $7.0\pm0.0$ \\
\midrule

600/N20 (D1)
& $\mathbf{0.462\pm0.086}$
& $0.503\pm0.111$
& $0.558\pm0.093$
& $0.223\pm0.023$
& $\mathbf{9.5\pm0.6}$
& $11.1\pm18.1$
& $39.7\pm58.4$
& $3.2\pm0.0$ \\

600/N40 (D2)
& $\mathbf{0.386\pm0.063}$
& $0.427\pm0.082$
& $0.408\pm0.081$
& $0.121\pm0.013$
& $\mathbf{11.0\pm0.6}$
& $80.2\pm57.4$
& $211.3\pm107.4$
& $6.4\pm0.0$ \\

600/N60 (D3)
& $\mathbf{0.235\pm0.043}$
& $0.328\pm0.066$
& $0.264\pm0.051$
& $0.080\pm0.007$
& $\mathbf{12.6\pm0.3}$
& $97.9\pm79.7$
& $204.8\pm153.4$
& $9.6\pm0.0$ \\

600/N80 (D4)
& $\mathbf{0.159\pm0.024}$
& $0.283\pm0.061$
& $0.175\pm0.026$
& $0.062\pm0.005$
& $\mathbf{15.1\pm0.7}$
& $120.4\pm73.8$
& $276.8\pm161.2$
& $12.8\pm0.0$ \\
\bottomrule
\end{tabular}
\end{table*}

\begin{table}[!t]
\centering
\caption{Fixed-area scaling beyond 80 robots: modeled payload and normalized covariance-trace AUC (mean $\pm$ sample standard deviation, 20 paired trials). Lower is better; Centralized denotes ideal centralized baseline.}
\label{tab:large_team_scaling}
\scriptsize
\setlength{\tabcolsep}{2.5pt}
\resizebox{\columnwidth}{!}{%
\begin{tabular}{c cc cc}
\toprule
& \multicolumn{2}{c}{Modeled payload [MB]}
& \multicolumn{2}{c}{Normalized covariance-trace AUC} \\
\cmidrule(lr){2-3}
\cmidrule(lr){4-5}
$N$ & AC-DC & Centralized & AC-DC & Centralized \\
\midrule
100 & $16.6\pm0.5$ & $16.0\pm0.0$ & $0.135\pm0.022$ & $0.050\pm0.006$ \\
120 & $19.3\pm0.3$ & $19.2\pm0.0$ & $0.129\pm0.045$ & $0.041\pm0.003$ \\
\bottomrule
\end{tabular}%
}
\end{table}

For the independent-cell specialization in \Cref{eq:sensing_information_model,eq:local_information}, we instantiate the diagonal measurement-information contribution using a finite-range Gaussian spatial sensitivity inspired by the sensor-footprint model in~\cite{ayvali2017ergodic}:
\[
[M_i(p)]_{rr}=\frac{w(r;p)}{\sigma_z^2},
\]
where
\[
w(r;p)=\exp\!\left(-\frac{\|r-p\|_2^2}{2\sigma_s^2}\right)
\mathbf 1_{\{\|r-p\|_2\le r_s\}}.
\]
Here $r_s$, $\sigma_s$, and $\sigma_z^2$ are the sensing radius, spatial-sensitivity width, and measurement-noise variance. We use $r_s=\SI{15}{m}$, $\sigma_s=\SI{5}{m}$, and $\sigma_z^2=0.01$ in all reported experiments.

The shared mission priority switches from $\rho_1$ to $\rho_2$ at $t_c=\SI{100}{s}$ via a one-time command outside the mesh and excluded from mesh traffic. The controller weights posterior variance above $V_{\mathrm{floor}}=0.01$ by priority and normalizes over the workspace, using normalized priority if residual mass vanishes; map-grid quadrature evaluates \Cref{eq:effective_target_coefficients}.

\noindent\textbf{Compared methods and common interface.}
Decentralized methods share two task signals and a controller. The trajectory input uses up to $40$ executed samples and a $30$-step nominal rollout with $\alpha_c=0.7$. Decentralized methods smooth $\hat c_i$ and ergodic mismatch with first-order gain $0.05$ at $20$~Hz; the ideal centralized baseline uses exact team quantities.

All decentralized methods share one radio timeline with $R_c=\SI{100}{m}$, $R_I=\SI{150}{m}$, $R_{\mathrm{bit}}=250~\mathrm{kbps}$, and eight TDMA slots per frame. A common persistent deficit allocator targets a $25/75$ trajectory/sensing-information airtime split, serves the largest normalized eligible deficit, retains signed deficits across frames, and alternates exact ties.

Here, ADMM-DAC uses asynchronous neighbor-specific $q_{ij}^{\mathrm{ADMM}}\in\mathbb R^{K_\mu}$ messages~\cite{bastianello2022admm_dac_imperfect}, PP-ACDC uses native packed all-sender broadcasts~\cite{makridis2025average}, and the ideal centralized baseline operates outside the constrained robot mesh.


\noindent\textbf{AC-DC configuration.}
We use one AC-DC parameter set in \Cref{tab:acdc_parameters} across all settings and trials without retuning. The gossip weight $\gamma=0.5$ gives pairwise averaging in \Cref{eq:pairwise_gossip}. Message granularity is fixed by regime: D1 uses complete-stream ordinary services, while D2--D4 use selected-block services from \Cref{eq:block_selection,eq:regime_message_coordinates}. Block and neighbor ties are resolved deterministically. We chose the remaining parameters empirically based on whether the communication policy operated as intended. Drift summaries are omitted when first-contact or ordinary data are sent.


\noindent\textbf{Communication accounting.}
We count all attempted modeled protocol messages, including failed final delivery checks, over both streams and all frames. 
We calculate AC-DC traffic with the fixed field sizes in \Cref{tab:protocol_accounting}; whereas \Cref{prop:communication_complexity} states scalable-identifier asymptotics.
ADMM-DAC sends a $4K_\mu$-byte unicast per active neighbor, while PP-ACDC sends one packed broadcast of $\left\lceil(42K_\mu+\delta_N)/8\right\rceil$ bytes per active sender, with $\delta_N=\max\{1,\lceil\log_2N\rceil\}$. 

\subsection{Metrics}\label{subsec:metric}
To evaluate information gathering independently of local DC tracking error, we use $I_i$, $\Omega_0$, and $\rho_P$ from \Cref{sec:problem} and pool the uncompressed cell-wise information offline. Under the independent-cell Gaussian map model, let $I_0(r)\triangleq[\Omega_0]_{rr}$; the pooled posterior variance for method $\nu$ is
\begin{equation}
V_{\mathrm{team}}^{(\nu)}(t,r)
=\left(I_0(r)+\sum_{i=1}^{N}I_i^{(\nu)}(t,r)\right)^{-1}.
\label{eq:ideal_team_variance}
\end{equation}
This pooled quantity is never provided to a decentralized controller. It evaluates the task-level uncertainty achieved by the realized closed-loop trajectories, whereas $\mathcal E_{\mathrm{erg}}$ in \Cref{eq:ergodic_mismatch} is the controller's spectral visitation mismatch.

Let $R^{(\nu)}(t,r)\triangleq[V_{\mathrm{team}}^{(\nu)}(t,r)-V_{\mathrm{floor}}]_+$ and $t_1\triangleq\Delta t$. Suppressing $\nu$ when clear, normalize the mission-priority-weighted remaining variance by its initial weighted value:
\begin{equation}
\mathcal W(t)
=\frac{\sum_{r\in\mathcal Q}\rho_P(t,r)R(t,r)\,\Delta A}
{\sum_{r\in\mathcal Q}\rho_P(t_1,r)R(t_1,r)\,\Delta A}.
\label{eq:evaluation_ratio}
\end{equation}
Because the mission-priority weights switch at $t_c$, $\mathcal W(t)$ may jump before the next sensing update. 
We compute the normalized covariance-trace $\mathrm{AUC}=\frac{1}{T-t_1}\int_{t_1}^{T}\mathcal W(t)\,dt$ by trapezoidal integration; lower is better. 
We call this the normalized covariance-trace AUC. Paired gain relative to baseline $\nu$ is $100(\mathrm{AUC}^{(\nu)}-\mathrm{AUC}^{(\mathrm{AC-DC})})/\mathrm{AUC}^{(\nu)}$.

Statistical summaries use 20 paired trials per setting for the main results and 10 for the message-granularity intervention. Across-setting gains average paired gains within each setting before averaging across settings.

\subsection{Results}\label{subsec:overall_results}\label{subsec:priority_response}\label{subsec:d1_results}

\noindent\textbf{Task performance.} \Cref{fig:remaining_variance_all,tab:main_results} summarize the 20-trial results. AC-DC has the lowest mean normalized covariance-trace AUC among the decentralized methods in all twelve settings. Averaging setting-wise paired gains across settings gives $29.5\%$ relative to ADMM-DAC and $24.5\%$ relative to PP-ACDC. The closest mean comparison is a near tie at 450/N44: $0.126$ for AC-DC versus $0.128$ for PP-ACDC. 

\noindent\textbf{Communication.}\label{subsec:baseline_complexity}\label{subsec:communication_usage} AC-DC has the lowest mean attempted modeled payload among the decentralized methods in all settings. ADMM-DAC and PP-ACDC use $5.8\times$ and $11.5\times$ as much payload as AC-DC on average, respectively. At 600/N80, AC-DC uses $15.1$~MB versus $120.4$ and $276.8$~MB, while the ideal centralized baseline uses $12.8$~MB. Thus, AC-DC approaches the traffic of the ideal centralized baseline while retaining peer-to-peer operation.


\noindent\textbf{Large-team fixed-area scaling.} 
To examine communication scaling beyond 80 robots, we test $N=100,120$ in the fixed $600\times600$~m environment with 20 paired trials. 
\Cref{tab:large_team_scaling} shows AC-DC approaching the ideal centralized baseline in modeled payload while remaining peer-to-peer.


\noindent\textbf{Message-granularity intervention.} Complete-stream messages improve AUC in most paired D1 trials under sparse contact, with a setting-dependent payload cost. A separate 10-trial intervention compares complete-stream and selected-block AC-DC. Complete-stream AUC is lower in $9/10$, $9/10$, and $8/10$ paired trials at 300/N5, 450/N11, and 600/N20, with paired gains of $22.8\%\pm20.1\%$, $18.1\%\pm11.2\%$, and $10.7\%\pm9.0\%$; the corresponding payload changes are $+182.8\%\pm32.8\%$, $+59.7\%\pm12.6\%$, and $+0.3\%\pm6.5\%$. 

\section{Conclusion}\label{sec:conclusion}

We presented AC-DC for communication-constrained decentralized ergodic search. AC-DC tracks team visitation for motion coordination and regional measurement information for uncertainty-map updates, while jointly deciding \emph{Who} communicates with whom, \emph{When}, and over \emph{What} parts of the consensus state. Successful services retain the DC backbone's average-preservation and pairwise-contraction properties.

In our comparisons, AC-DC has lower mean normalized covariance-trace AUC and attempted modeled payload than ADMM-DAC and PP-ACDC. In fixed-area tests, AC-DC approaches the ideal centralized baseline in modeled payload.
Future work will infer communication regimes online, adapt state granularity, and establish tracking guarantees under explicit connectivity and delivery assumptions.

\bibliographystyle{IEEEtran}
\bibliography{references}

\end{document}